\documentclass[a4paper]{llncs}
\usepackage[english]{babel}
\usepackage[utf8]{inputenc}
\usepackage{multicol}

\usepackage[]{inputenc}

\usepackage{graphicx}
\usepackage{amsmath}
\usepackage{amsfonts}
\usepackage{amssymb}
\usepackage{mathptmx}
\usepackage{graphicx}
\usepackage{algpseudocode}
\usepackage{algorithm}
\usepackage{subfigure}

\usepackage{comment}

\newtheorem{mydef}{Definition}
\newtheorem{theo}{Theorem}

\newtheorem{col}{Corollary}
\newtheorem{prop}{Proposition}

\begin{document}
\title{Relational inductive biases on attention mechanisms\thanks{Original version published in Spanish into the journal \emph{Rsearch in Computing Science} \url{https://www.rcs.cic.ipn.mx/} }}
\titlerunning{~}

\author{V\'ictor Mijangos$~^1$, Ximena Gutierrez-Vasques$~^2$, Ver\'onica E. Arriola$~^1$, Ulises Rodr\'iguez-Dom\'inguez$~^1$, Alexis Cervantes$~^1$, Jos\'e Luis Almanzara$~^1$}
%
\institute{$~^1$Facultad de Ciencias, UNAM;  $~^2$CEIICH, UNAM 
\\
\email{\{vmijangosc, v.arriola, ulises.rodriguez.dominguez, alexis.cervantes, jose-luis\}@ciencias.unam.mx}, xim@unam.mx}

\maketitle              

\begin{abstract}
Inductive learning aims to construct general models from specific examples, guided by biases that influence hypothesis selection and determine generalization capacity. In this work, we focus on characterizing the relational inductive biases present in attention mechanisms, understood as assumptions about the underlying relationships between data elements. From the perspective of geometric deep learning, we analyze the most common attention mechanisms in terms of their equivariance properties with respect to permutation subgroups, which allows us to propose a classification based on their relational biases. Under this perspective, we show that different attention layers are characterized by the underlying relationships they assume on the input data.

\end{abstract}

\section{Introduction}

One of the most common paradigms in current machine learning methods is \textit{inductive learning}, whose goal is to construct a general function, or hypothesis, from observed particular examples \cite{mitchell1980need,mitchell1997machine}. In this way, given a training set (input-output pairs), the model seeks a hypothesis, within a space of possible hypotheses, that fits the training data well and can generalize to instances that have not been seen before.

The assumptions made by a learning algorithm to propose hypotheses, defined over the entire problem domain and not only the values of the observed instances, constitute an \textit{inductive bias}. These assumptions are what give the model the potential to generalize to unseen data \cite{mitchell1997machine}. Another form of inductive bias includes assumptions that prefer certain hypotheses over others. Therefore, inductive biases play a key role in the generalization ability of machine learning models.

Although inductive biases are one of the components that enable learning in various tasks (e.g., natural language processing, image recognition and generation) it is difficult to find works that formally address the foundations of inductive biases, particularly in deep learning.

In this paper, we focus on explaining the operation of one of the most prominent types of inductive biases, namely \emph{relational biases}, which exploit the inherent relational structure of data \cite{battaglia2018relational}. Our analysis is applied to transformer neural networks and their attention mechanisms. We propose adopting strategies from geometric deep learning to describe relational biases, as they provide a robust framework for modeling relational structures in data, as well as transformations and symmetries. Our main contributions are: i) to propose a theoretical framework of relational inductive biases through geometric deep learning; ii) to formally show that attention layers in transformers are part of the graph neural network (GNN) framework; iii) to prove the inductive biases assumed by the most common attention layers in literature; iv) to show that the graph structure can be implemented through masking; and v) to propose a hierarchy of attention mechanisms. 

\section{Theoretical Framework}

\subsection{Relational Inductive Biases}

Inductive biases are \textit{a priori} assumptions that help learning algorithms to choose one hypothesis over another \cite{mitchell1980need}. Tom Mitchell \cite{mitchell1997machine} suggests posing the question of what \textit{a priori} assumptions are needed for the learning algorithm to carry out a deductive process to generalize over a new instance $x$. These \textit{a priori} assumptions are precisely what is meant by inductive biases.

\begin{mydef}[Inductive Bias] \label{def:InductiveBias}
Let $f^*: X \to Y$ be an arbitrary target function, and let $D = \{(x,f^*(x)) : x \in X\}$ be a training dataset, and $A$ a learning algorithm. An inductive bias of $A$ is a minimal set of assumptions $B$ such that for all $x \in X$:\footnote{We use Mitchell’s notation \cite{mitchell1997machine}, taking $B$ and $D$ as conjunctions over their elements viewed as propositions.}
$$(B \land D \land x) \vdash A(x, D)$$
Where $A(x, D)$ is the prediction of the algorithm trained on $D$ for instance $x$, and $\vdash$ denotes the inference of $A(x,D)$ from $(B \land D \land x)$.
\end{mydef}

In particular, we are interested in relational inductive biases \cite{battaglia2018relational} to analyze data composed of sets of entities that are structurally related to each other. For example, let $X$ be a set of sentences in Spanish, each instance $x$ is made up of the set of words ${x_1,...,x_n}$ in the sentence.

\begin{mydef}[Relational Inductive Bias]
Given a set of entities that make up an instance $x = \{x_1,...,x_n\}$, a relational inductive bias is an assumption about the relations among those entities. That is, it is a relational structure $(x,G)$, such that $G = (V,E)$ is a graph relating the elements of $x$.
\end{mydef}

\subsection{Geometric Deep Learning}

To study relational inductive biases within \textbf{attention mechanisms}, it is crucial to adopt a theoretical framework that allows us to characterize and relate the different mechanisms that have been proposed. Bronstein et al. \cite{bronstein2021geometric} adopt an approach based on geometric features and their symmetry groups. Papillon et al. \cite{papillon2023architectures} have extended this approach to a topological perspective. Gavranović et al. \cite{gavranović2024categoricaldeeplearning} propose a more general framework based on category theory to encompass both geometric and topological perspectives (which they call “top-down”) and “bottom-up” frameworks, which are based on constructing architectures via automatic differentiation methods \cite{abadi2016tensorflow}, \cite{paszke2019pytorch}, \cite{bradbury2018jax}. We focus on the relations established by a graph assumed within attention models. We adopt the theory of geometric deep learning \cite{bronstein2017geometric}, since this theoretical model, by focusing on geometric structures in the data domain, is ideal for expressing relationships between data entities.

Geometric deep learning adopts a study methodology based on Felix Klein’s program for geometry \cite{kisil2012erlangen}: a set or domain of elements is defined along with groups of transformations associated with this domain. Based on this idea, the goal of geometric deep learning is to determine characteristics (mainly relational) of the data domain and to determine the type of hidden layers that can leverage these characteristics.

The essential idea behind geometric deep learning lies in studying instances $x$ and the relationships among their entities $x_i$. These are usually represented by graphs $G = (V,E)$ where the vertices in $V$ are associated with entities and $E$ represents their relationships. The data domain and the task determine the types of layers in a deep model. These layers must preserve the relevant information needed to solve the tasks. To formalize this, the concept of an equivariant function is key.\footnote{A particularly important example is convolutional layers. In these, the domain consists of images, which can be seen as elements $x \in \mathbb{R}^{H\times W \times C}$, where $H$ is height, $W$ is width, and $C$ is the number of channels. The action of a translation $g \in \mathcal{G}$ can be represented by a translation matrix $T = \rho(g)$, such that $Tx$ is a translated version of the image $x$. A convolution $f: \mathbb{R}^{H\times W \times C} \to \mathbb{R}^{H\times W \times C}$ is equivariant to translations if the equality $f(Tx) = Tf(x)$ holds.}

\begin{mydef}[Equivariant Function] \label{def:equivariantFunction}
Let $\mathcal{G}$ be a symmetry group over a set $X$, a function $f:X \to X$ is said to be $\mathcal{G}$-equivariant if for every group action $g \in \mathcal{G}$ the following holds:
\begin{equation} \label{eq;equivariance}
f\big( \rho(g)x \big) = \rho(g) f(x)
\end{equation}
Where $\rho(g)$ is the (matrix) representation of the action $g$.
\end{mydef}

Another important concept within this theoretical framework is that of invariance and invariant functions, where $f\big(\rho(g)x\big) = f(x)$.
Given the characteristics of attention mechanisms, we focus on $\mathcal{G}$-equivariant functions, or simply equivariant functions. The geometric deep learning framework allows for defining many types of layers, mainly for graph neural network architectures. We will see that attention layers fall under the concept of a graph layer.

\begin{mydef}[Graph Layer] \label{def:GraphLayer}
A graph layer (or message-passing layer) depends on a graph $G = (V,E)$ with a neighborhood system $\mathcal{N}_v$ for each vertex associated with an entity $x_v$, such that the hidden representation of that entity $h_v$ has the form:
\begin{equation} \label{eq:InvEqFunction}
h_v = \phi\Big( x_v, \bigoplus{u \in \mathcal{N}_v} \psi(x_v, x_u) \Big)
\end{equation}
Where the following elements are distinguished:
\begin{enumerate}
\item Message function $\psi(x_v, x_u)$: Generates a message (usually a vector) based on $x_v$ and its neighbors $x_u$.
\item Aggregation function $\bigoplus$: Determines how messages are combined for the update; it is required to be a commutative operator.
\item Update function $\phi(\cdot, \cdot)$: Defines the final form of the new representation $h_v$ based on the original representation $x_v$ and the relational structure.
\end{enumerate}
\end{mydef}

In Section~\ref{sec:results}, we study the equivariances of attention mechanisms using subgroups of finite permutations $S_n$. To this end, we take as a starting point Cayley’s theorem, which states that every group is isomorphic to a permutation subgroup \cite{cayley1854vii}.

\subsection{Attention Mechanisms} \label{sec:AttentionMech}

Attention mechanisms were introduced for the problem of machine translation in recurrent networks by Bahdanau et al. \cite{bahdanau2014neural}.
Later, Vaswani et al. \cite{vaswani2017attention} proposed replacing recurrences with attention mechanisms that operate on the same input data, which they called self-attention.
Following these works, new architectures and attention layers have been defined, of which we present the most common ones here, starting with self-attention \cite{vaswani2017attention}.

\begin{mydef}[Self-Attention] \label{def:selfAttention}
Given a set of entities from an instance $\{x_1,..,x_n\}$ with $x_i \in \mathbb{R}^d$, a self-attention mechanism is a layer of the form:
\begin{equation} \label{eq:selfAttention}
h_i = \sum_{j} \alpha(x_i, x_j) \psi_v(x_j)
\end{equation}
where $\alpha(x_i, x_j)$ are the attention weights defined as:
\begin{equation} \label{eq:attentionWeights}
\alpha(x_i, x_j) = Softmax\Big(\frac{\psi_k(x_j)^T \psi_q(x_i)}{\sqrt{d}}\Big)
\end{equation}
We denote $\psi_q, \psi_k$, and $\psi_v$ as the projections applied to the input entities, generally defined by a linear or affine function.
\end{mydef}

Transformers \cite{vaswani2017attention} integrate a type of attention similar to that of Bahdanau \cite{bahdanau2014neural} which represents input entities (in the encoder) with output entities (in the decoder). This encoder-decoder attention can be defined using the concepts of Definition~\ref{def:selfAttention}, restricting how the relationships are determined.

\begin{mydef}[Encoder-Decoder Attention] \label{def:encDecAttention}
Suppose a bipartition determined by the sets of entities $X = \{x_1,...,x_m\}$ and $Y = \{y_1,..., y_n\}$. We define encoder-decoder attention as the mechanism that, for each $y_i$, with $i \in \{1,...,n\}$, obtains hidden representations as:
\begin{equation} \label{eq:encDecAttention}
h_i = \sum_{j=1}^m \alpha(y_i, x_j) \psi_v(x_j)
\end{equation}
Where $\alpha(y_i,x_j)$ is defined similarly to Equation~\ref{eq:attentionWeights}.
\end{mydef}

When attention mechanisms are used to predict an element (entity) given a previous set (as in text generation), the training of these models cannot assume that previous elements depend on future ones, which are not yet known. To handle this, decoders in transformers implement masked attention mechanisms \cite{vaswani2017attention}.

\begin{mydef}[Masked Attention] \label{def:maskedAttention}
Given an input set of ordered entities $x_1,..,x_n$, a masked attention mechanism is a layer that obtains representations of the form:
\begin{equation} \label{eq:maskedAttention}
h_i = \sum_{j \leq i} \alpha(x_i, x_j) \psi_v(x_j)
\end{equation}
where the attention weights $\alpha(x_i, x_j)$ are estimated as in Equation~\ref{eq:attentionWeights}.
\end{mydef}

In the previous definition, we have imposed an order to simplify the expression in the summation.
The idea of disconnecting nodes within the graph that defines relationships among entities in an attention mechanism can be extended further. For example, Child et al. \cite{child2019generating} suggest a strided attention mechanism where the relationships are bounded by various constraints, such as a limit on how many previous elements an entity can connect to.

\begin{mydef}[Strided Attention] \label{def:strideAttention}
Strided attention is a type of sparse attention where, given an input of ordered entities $x_1,...,x_n$, their representations are obtained as:
\begin{equation}
h_i = \sum_{t \leq j \leq i} \alpha(x_i, x_j) \psi_v(x_j)
\end{equation}
where $t = \max\{0,i-k\}$ for some constant $k$, and $\alpha(x_i, x_j)$ is defined as in Equation~\ref{eq:attentionWeights}.
\end{mydef}

Another general way to define attention mechanisms is to assume that the relationships among entities in each instance are arbitrary and are defined by the adjacency matrix of a graph \cite{velickovic2018graphattentionnetworks}.

\begin{mydef}[Graph Attention] \label{def:graphAttention}
Given entities $\{x_1,..,x_n\}$ of a data point along with information about their neighbors $\mathcal{N}_i$ for all $i$, a graph attention mechanism is a layer of the form:
\begin{equation} \label{eq:graphAttention}
h_i = \sum_{j \in \mathcal{N}_i} \alpha(x_i, x_j) \psi_v(x_j)
\end{equation}
where $\alpha(x_i, x_j)$ are the attention weights as in Equation~\ref{eq:attentionWeights}.
\end{mydef}

Sparse attention and graph attention differ in that sparse attention assumes arbitrary relationships for all instances in the data domain, whereas in graph attention, the relationships depend on each particular data point. The latter mechanism is the most general form of attention mechanism, from which the previously defined mechanisms can be derived.

\subsection{Attention as a Generalized Kernel}

Tsai et al. \cite{tsai-etal-2019-transformer} propose classifying attention mechanisms based on a kernel centered on the function $\alpha(x_i, x_j)$ from Equation~\ref{eq:attentionWeights}.
From this perspective, attention weights depend on a function $k:X \times X \to \mathbb{R}$, where $X$ is the feature space for the attention mechanisms (including both positional and non-positional features \cite{vaswani2017attention}). In this framework, $k(\cdot, \cdot)$ is a kernel, which in attention mechanisms is exponential, and since the projections $\psi_q, \psi_k$ are generally not symmetric, $k$ is considered a generalized (non-symmetric) kernel.

The kernel-based view \cite{tsai-etal-2019-transformer} does not account for the inductive biases that may exist within these mechanisms.
In what follows, we assume that the attention weights $\alpha(x_i, x_j)$ are determined by the softmax over the kernel values $k(x_i,x_j)$ under a generalized kernel. That is:
\begin{equation} \label{eq:KernelAttention}
\alpha(x_i, x_j) = Softmax_K\big( K_{i,j} \big)
\end{equation}
We do not delve deeply into the consequences of using different types of kernels \cite{tsai-etal-2019-transformer}, since our proposal focuses on relational inductive biases.
However, it is worth noting that this perspective opens the possibility for relational biases that are not necessarily binary. This is a direction we leave for future work.
We emphasize that our proposal is not contrary but complementary to the view of these mechanisms within the framework of generalized kernels.

\section{Results} \label{sec:results}

From the review of different attention mechanisms (Section~\ref{sec:AttentionMech}), we observed that these mechanisms share a general core, involving an additive aggregation of neighboring entity values weighted by a probability typically estimated through a softmax function. This already introduces a relational inductive bias within attention mechanisms.

\begin{prop}[Stochastic Relation Bias]
Attention layers assume that the relationships among entities $x_1,...,x_n$ are stochastic.
\end{prop}

The attention matrix, which we denote as $\alpha$, defines the associated stochastic matrix, where each entry $\alpha_{i,j} := \alpha(x_i,x_j) = p(x_j|x_i)$. This probability is given in terms of the softmax function (Equation~\ref{eq:KernelAttention}).

It is worth noting that attention mechanisms fall within layers of the form presented in Equation~\ref{eq:InvEqFunction} (Definition~\ref{def:GraphLayer}): 
\begin{enumerate}
    \item  the message function is defined as $\psi(x_i, x_j) = \alpha(x_i, x_j) \psi_v(x_j)$;
    \item aggregation is carried out by summing over all entities of a data point; and
    \item the update function obtains the new representations through the aggregation of messages.
\end{enumerate}
Highlighting the graphical and stochastic framework, we propose a general definition of an attention layer.

\begin{mydef}[Attention Layer]
An attention layer is a type of graph layer (Definition~\ref{def:GraphLayer}) that estimates the representation of a set of entities $\{x_1, x_2,...,x_n\}$ with neighborhood system $\{\mathcal{N}_i : i=1,2,...,n\}$ from the expected value over a distribution $p$ of these neighborhoods:
$$h_i = \mathbb{E}_{p\sim \mathcal{N}_i}\big[ \psi_v(x) \big]$$
Where the expectation is estimated over the relationships of an adjacency matrix given as in Equation~\ref{eq:KernelAttention}:
$$\alpha(x_i, x_j) = Softmax_K(K_{i,j})$$
$\psi_v(x)$ is a projection of the data into the value space.
\end{mydef}

In the following, we rely on this definition to demonstrate the relational inductive biases reflected in each type of attention mechanism. In particular, we focus on how these relationships are expressed in the attention matrix, since, as noted, this matrix determines the relationships among entities. To do so, we concentrate on masking processes.

\begin{lemma} \label{lem:Masked}
The underlying relationships in the data domain of an attention mechanism are manifested in the attention matrix as a masking process.
\end{lemma}

\begin{proof}
The masking process consists of removing certain entries from the attention matrix $\alpha$. Let $K$ be the product matrix (Equation~\ref{eq:KernelAttention}) to which the softmax function is applied. An entry is masked by assigning $K_{i,j} = -\infty$ before applying the softmax, resulting in $\alpha_{i,j} = 0$. Clearly, this represents a disconnection in the underlying relational graph structure. If $G = (V,E)$, we can define the masking process as:
$$K_{i,j} = \begin{cases}
        k(x_i,x_j) & \text{ if } (i,j) \in E \\
        -\infty & \text{ if } (i,j) \not\in E \\
    \end{cases}$$
Thus, the masking is determined by the assumed relationships in the data domain.
\end{proof}

\begin{theo}[Relational bias in self-attention] \label{bias:selfAttention}
A self-attention layer assumes a relational inductive bias based on a fully connected graph \cite{bronstein2021geometric}.
\end{theo}

\begin{proof}
In a self-attention layer (Definition~\ref{def:selfAttention}), the aggregation—typically applied over neighbors in the graph of the node representing the entity $x_i$ in message-passing networks—is instead applied over all input entities.

Based on Lemma~\ref{lem:Masked}, if $K$ is the product matrix, then for all entities $x_i, x_j$ from the input data we have $K_{i,j} = k(x_i,x_j) > -\infty$, so $\alpha_{i,j} > 0$ in the attention matrix. This implies that no disconnections occur in the underlying graph; i.e., a fully connected graph is assumed.

\end{proof}

\begin{figure}[ht]
\centering
\subfigure[Attention Matrix]{\includegraphics[width=0.45\linewidth]{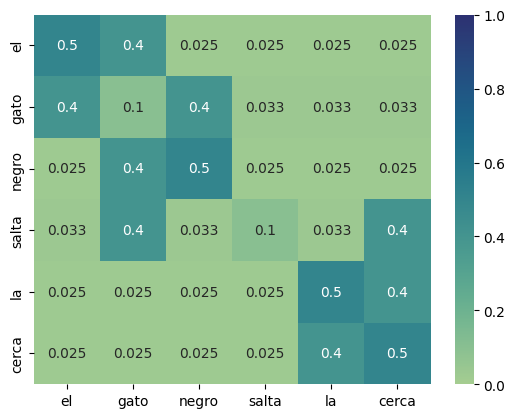}}
\subfigure[Relational Graph]{\includegraphics[width=0.45\linewidth]{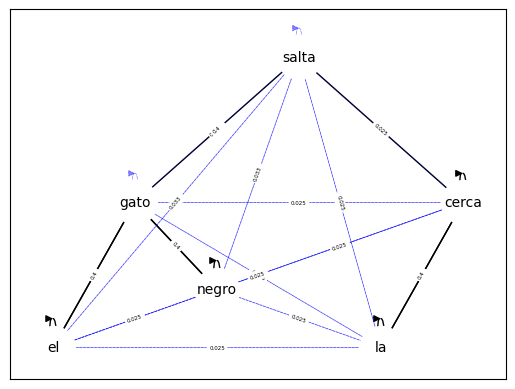}}
\caption{Example of relationships established by a self-attention mechanism}
\label{fig:AttentionGraphRepresentation}
\end{figure}

Therefore, self-attention layers assume a relational inductive bias where all entities in the data are connected to each other, i.e., $E = x \times x$ (see Figure~\ref{fig:AttentionGraphRepresentation}).
As for encoder-decoder attention layers (Definition~\ref{def:encDecAttention}), these are based on aggregating over input entities and are aimed at representing output entities only.

\begin{theo}[Relational bias in encoder-decoder attention] \label{theo:encDecRel}
Encoder-decoder attention layers assume a relational inductive bias based on a bipartite graph.
\end{theo}

\begin{proof}
In an encoder-decoder attention layer, we have a set of entities: $$\{x_1, x_2,...,x_n, y_1,y_2,...,y_m\}$$ with a partition into sets $X = \{x_1,...,x_n\}$ and $Y = \{y_1,...,y_m\}$ where $X \cap Y = \emptyset$. The product matrix $K$ is defined as follows:
$$K_{i,j} = \begin{cases}
        k(y_i, x_j) & \text{ if } y_i \in Y \land x_j \in X \\
        -\infty & \text{ otherwise}
    \end{cases}$$
That is, $\alpha_{i,j} > 0$ if and only if $x_i \in X$ and $y_j \in Y$, meaning the attention matrix only connects the lower-left block (Figure~\ref{fig:AttentionGraphEncDec}).
\end{proof}

The partition is determined by the input and output entities; typically, this is represented as a directed graph from inputs to outputs. This type of attention can be applied to both transformers and recurrent networks.

\begin{figure}[ht]
\centering
\subfigure[Attention Matrix]{\includegraphics[width=0.45\linewidth]{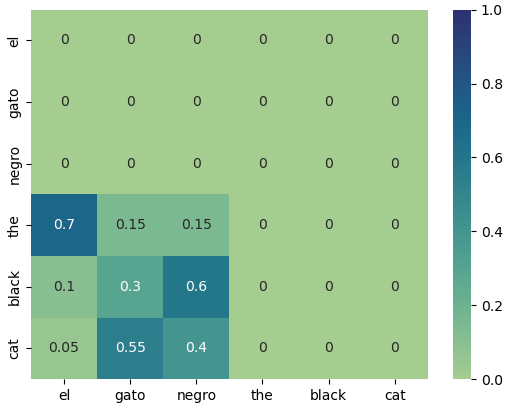}}
\subfigure[Relational Graph]{\includegraphics[width=0.45\linewidth]{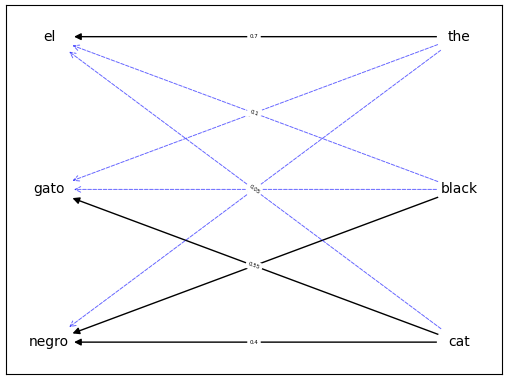}}
\caption{Example of relationships established by an encoder-decoder attention mechanism}
\label{fig:AttentionGraphEncDec}
\end{figure}

The third mechanism to examine is masked attention (Definition~\ref{def:maskedAttention}). These attention type assumes an order among the input entities such that, when viewed as a graph, a connection between two nodes exists if and only if one entity precedes the other in this order.

\begin{theo}[Relational bias in masked attention] \label{bias:maskAttention}
Masked attention layers assume a relational inductive bias based on a total order.
\end{theo}

\begin{proof}
Let $x = \{x_1, x_2,...,x_n\}$ be the set of input entities, and define an order $O(x) = \{ (x_i,x_{i+1}) : i =1,...,n-1 \}$ over the entities. Considering the product matrix $K$ in the attention mechanism, masking is determined by:
$$K_{i,j} = \begin{cases}
        k(x_i, x_j) & \text{ if } j \leq i \\
        -\infty & \text{ if } j > i
    \end{cases}$$
Thus, $\alpha_{i,j} \neq 0$ if and only if $x_j$ precedes $x_i$ in $O(x)$, while other entries represent disconnections. Clearly, the nonzero entries appear in the lower triangular part of the attention matrix (Figure~\ref{fig:maskSparseAtt}).

Since the following hold: (i) for all $i$, $i \leq i$ by identity; (ii) if $j \leq i$ and $i \leq j$, then $i = j$ (otherwise the upper triangle would also have connections); and (iii) if $k \leq j$ and $j \leq i$, then $k \leq i$, we conclude that the order is total.    

\end{proof}

The types of relationships in masked attention layers introduce a relational inductive bias that assumes entities do not connect to subsequent elements. This is, as is well known, useful for sequence modeling.

\begin{figure}[ht]
\centering
\subfigure[Attention Matrix]{\includegraphics[width=0.45\linewidth]{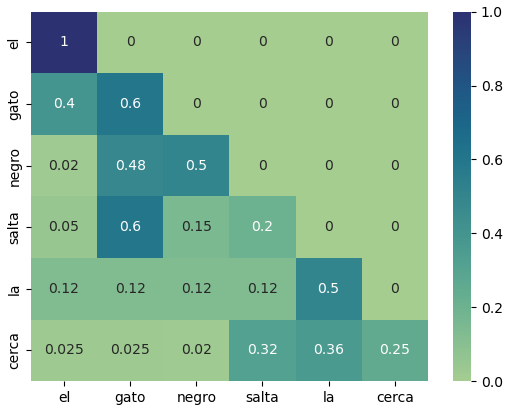}}
\subfigure[Relational Graph]{\includegraphics[width=0.45\linewidth]{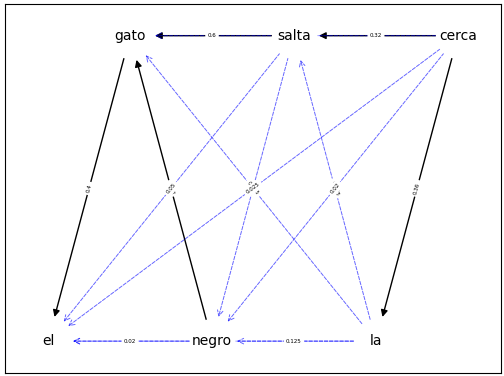}}
\caption{Example of relationships established by a masked attention mechanism}
\label{fig:maskSparseAtt}
\end{figure}

Both masked attention and stride sparse attention (Definition~\ref{def:strideAttention}) can be viewed as mechanisms that relate each entity only to previous ones. This defines a directed graph whose adjacency matrix has nonzero values only in the lower triangle. In stride attention, zeros also appear in the lower triangle but are limited to a fixed number of previous elements (see Figure~\ref{fig:stepAttention}).

\begin{theo}[Relational bias in stride attention] \label{bias:strideAttention}
Stride attention layers assume an inductive bias where an element connects to the $p$ previous elements for a fixed $p$ and a given order.
\end{theo}

\begin{figure}[ht]
\centering
\subfigure[Attention Matrix]{\includegraphics[width=0.45\linewidth]{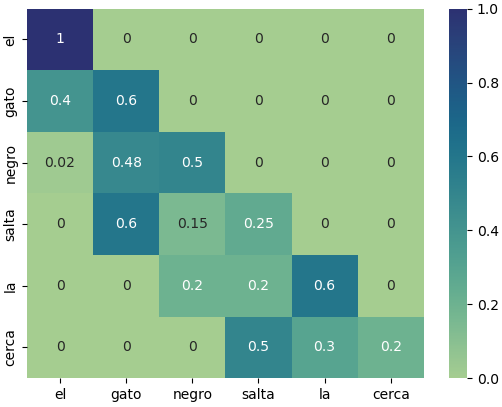}}
\subfigure[Relational Graph]{\includegraphics[width=0.45\linewidth]{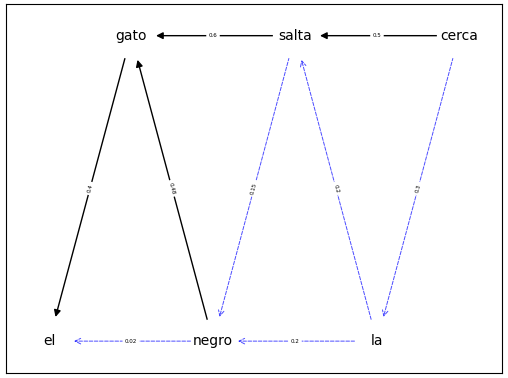}}
\caption{Example of relationships established by a strided attention mechanism}
\label{fig:stepAttention}
\end{figure}

\begin{proof}
Stride attention is similar to masked attention and also requires an order $O(x)$ over the entities. Under this order, we define the entries of $K$ as:
$$K_{i,j} = \begin{cases}
        k(x_i, x_j) & \text{ if } t \leq j \leq i \\
        -\infty & \text{ if } t > j > i
    \end{cases}$$
Here, $t = \max\{0, i-p\}$ for some $p \geq 0$. Relationships among entities occur only with the $p$ previous elements, defining a directed acyclic graph (DAG) that determines the relational inductive bias in these mechanisms.
\end{proof}

\subsection{Equivariance Sensitivity}

To analyze the types of equivariances that attention layers are sensitive to, we rely on permutation subgroups \cite{cayley1854vii}.
For self-attention layers, we have the following result already presented in \cite{bronstein2021geometric}.

\begin{theo}[Equivariance in self-attention] \label{equiv:selfAttention}
Self-attention layers are permutation equivariant.
\end{theo}

In masked and stride attention, there is a dependence on a predefined order.

\begin{theo}[Equivariance in masked attention]
Masked attention layers are equivariant to translations.
\end{theo}

\begin{proof}
A translation is a function $\sigma(i) = i + m$ for some constant $m$.
As shown earlier, masked attention layers assume an order $O(x)$ over the input entities to create the attention matrix $\alpha$ through the product matrix $K$. Under translation:
$$K_{\sigma(i),\sigma(j)} = \begin{cases}
        k(x_{\sigma(i)}, x_{\sigma(j)}) & \text{ if } \sigma(j) \leq \sigma(i) \\
        -\infty & \text{ if } \sigma(j) > \sigma(i)
    \end{cases}$$
This preserves the order (if $j \leq i$, then $\sigma(j) \leq \sigma(i)$).
\end{proof}

\begin{theo}[Equivariance in stride attention] \label{equvi:strideAttention}
Stride attention layers are equivariant to translations.
\end{theo}

\begin{proof}
Theorem~\ref{bias:strideAttention} introduced an order over the entities, which does not define a total order. For a fixed $p$, if $j < i - p$, then entity $j$ does not relate to $i$, although it may relate to entities that are related to $i$, so transitivity doesn't hold. However, stride attention defines a DAG in which a unique topological order can be defined, making it a total order. Thus, as in masked attention, a translation preserves this total order and does not alter relationships.
\end{proof}

\begin{col}
The relational inductive biases of stride attention layers subsume those of masked attention.
\end{col}

Stride attention generalizes masked attention by choosing $t = \max\{n : n = |x|\}$, i.e., considering all previous elements.

\begin{theo}[Equivariance in encoder-decoder attention]
Encoder-decoder attention is block permutation equivariant.
\end{theo}

\begin{proof}
A block permutation is an epimorphism $\sigma: X \to X$ such that, given an equivalence relation $\sim$, if $x \sim y$ in $X$, then $\sigma(x) \sim \sigma(y)$.
As noted in Theorem~\ref{theo:encDecRel}, encoder-decoder attention assumes a partition $X \cup Y$ of the data. Therefore, there exists an equivalence relation $\sim$ underlying this partition.
Clearly, if $x_i,x_j \in X$ then $x_i \sim x_j$, so $\sigma(x_i) \sim \sigma(x_j)$ and thus $\sigma(x_i), \sigma(x_j) \in X$. Similarly, permutations over $Y$ remain in $Y$. The sets $X$ and $Y$ remain disjoint under such permutations, preserving the partition.
\end{proof}

\subsection{Discussion}

We propose a hierarchy of common attention mechanisms based on relational biases. Other works have characterized transformers by architectural types \cite{liu2023survey}, \cite{lu2022transformers} without delving into attention-specific features. The work by Tsai et al. \cite{tsai-etal-2019-transformer} proposes a methodology for choosing attention mechanisms but does not address how different types of masking relate or what biases they introduce.

Table~\ref{tab:ClasificacionMecanismoAtt} shows the classification of attention mechanisms from the most general to the most specific, based on relational bias and equivariances.
Graph attention is more general than sparse attention.
Self-attention connects all entities to each other, while masked attention allows connections only to previous elements, and stride attention limits this to a fixed number of previous elements. Finally, encoder-decoder attention requires a bipartition.\footnote{An implementation of different attention layers and transformers can be found at: https://victormijangosdelacruz.github.io/MecanismosAtencion/.}

\begin{table}[!h]
\small
\centering
\begin{tabular}{c | p{2.7cm} p{3cm} p{3cm}} \hline
\textbf{Attention Mechanism} & \textbf{Relations} & \textbf{Symmetry Group} & \textbf{Data Type} \\ \hline
Graph Attention & Depends on the data & Dynamic & With graph structure for each instance \\
Sparse Attention & Arbitrary & Arbitrary & With specific graph relationships \\
Self-Attention & Fully connected & Permutations & Relational or bidirectional \\
Masked & Previous elements & Translation & Sequential \\
Strided & $p$ previous elements & Translation & Bounded sequential \\
Encoder-Decoder & Bipartition & Block system & Partitional \\ \hline
\end{tabular}
\caption{Classification of common attention mechanisms}
\label{tab:ClasificacionMecanismoAtt}
\end{table}

The previous table also includes the type of data that these models are suited for.
Graph attention is useful for datasets where each instance has an independent graph structure; for example, it has been applied to node classification and 3D surfaces described by triangulations, such as in the GAT model \cite{velickovic2018graphattentionnetworks}.
Self-attention networks are often used in auto-encode language models like BERT \cite{devlin2019bert}, where bidirectional relationships are assumed.
Masked attention is used in autoregressive models like GPT \cite{dai-etal-2019-transformer}, assuming future elements are unknown, allowing only backward relations.
Sparse attention is used in models like Unlimiformer \cite{bertsch2023unlimiformerlongrangetransformersunlimited}, which processes text blocks (limiting connections by proximity) and is applied in text classification and analysis.
Models like ViT \cite{dosovitskiy2021imageworth16x16words}\cite{Li_etal_ViT_explanation_attempt} or S4 \cite{gu2022efficientlymodelinglongsequences} use stride attention for processing images as token sequences, transforming image sections or patches.
Finally, T5 \cite{raffel2020exploring} incorporates encoder-decoder attention between input and output partitions, and has been applied to translation and summarization. Table~\ref{tab:models} summarizes these architectures and their applications.

\begin{table}[h]
\small
\centering
\begin{tabular}{p{2cm}|p{2.5cm} p{7.5cm}}
\hline
\textbf{Model} & \textbf{Type of Attention} & \textbf{Applications/Effects} \\
\hline
GAT \cite{velickovic2018graphattentionnetworks} & Graph-based & Node classification, clustering, recommendation systems, 3D models. \\
BERT \cite{devlin2019bert} & Self-attention & Sentiment analysis, text classification, question answering. \\
GPT \cite{Radford2018ImprovingLU} & Masked & Text generation, machine translation, automatic summarization. Enables generative pretraining of a large language model across multiple tasks in an unsupervised manner. \\
Unlimiformer \cite{bertsch2023unlimiformerlongrangetransformersunlimited} & Sparse attention & Text analysis, text classification, natural language processing. Enables learning of long-range sequences with limited computational resources. \\
ViT \cite{dosovitskiy2021imageworth16x16words} & Strided & Image classification, object detection, image segmentation. Captures spatial attention patterns that are both agnostic (e.g., in early layers) and content-sensitive (e.g., in deeper layers). \\
T5 \cite{raffel2020exploring} & Encoder-decoder & Machine translation, automatic summarization, text generation. Enables multitask training of the model in an unsupervised fashion. \\
\hline
\end{tabular}
\label{tab:models}
\caption{Attention-based models for language and vision tasks}
\end{table}

One of the main limitations of transformers is that their operation relies on attention mechanisms that assume fully connected graphs, which requires large amounts of data to adequately learn meaningful relationships. Introducing constraints on the possible relationships can embed a relational bias that may improve the performance of architectures based on these layers, provided that such inductive biases align with the structure of the data.

\section{Conclusions and Future Work}

In this work, we presented a theoretical approach to relational inductive biases within attention mechanisms. To this end, we relied on the theory of geometric deep learning \cite{bronstein2021geometric}, which allowed us to study the types of permutation subgroups under which these mechanisms are equivariant. We thus developed a classification based on the assumptions each mechanism makes about the underlying relationships among the entities in the data domain. Our analysis provides a deeper understanding of how attention mechanisms operate and how they can process complex phenomena, such as syntactic variations in natural languages.

We also extended the characterization of masking and relationships in attention mechanisms by Tsai et al. \cite{tsai-etal-2019-transformer}. However, in attention mechanisms, relationships occur not only between pairs of entities (a common graph structure) but may also involve higher-order relationships (triplets, quadruplets, etc.). The nonlinearities in these mechanisms, such as those evident from the perspective of generalized kernels \cite{tsai-etal-2019-transformer}, can be a source of such higher-order relationships. For example, in the recent work by \cite{Rao_etal_2022}, it was shown that visual transformers consider high-order spatial interactions within each basic block of their layers. In \cite{Hao_etal_2021}, the interactions between text tokens in a transformer model are interpreted by constructing hierarchical relationships, showing that such interactions do not occur merely in pairs but at higher orders.

As future work, a key challenge is to thoroughly investigate the higher-order relationships that can arise in these mechanisms to achieve a complete characterization of them. Another important aspect is to explore the connections with other types of layers, such as convolutional ones, which can be seen as a subset of attention layers \cite{bronstein2021geometric}. This may provide a more comprehensive understanding of attention mechanisms and their role in transformer architectures.

\section*{Acknowledgments}

We would like to thank the reviewers for their comments, which will help improve the quality of this work. We also wish to thank the PAPIIT projects TA100924 and TA100725 from UNAM.

\bibliographystyle{splncs04}
\bibliography{main}

\newpage
\appendix
\section{Attention-Based Architectures} \label{app:appendix}

\subsection{BERT and Self-Attention}

The BERT model \cite{devlin2019bert} was one of the first language models to leverage the transformer architecture. The goal of this model was to create word embeddings that could be applied to general tasks. These models are commonly referred to as encoder-only language models and contrast with autoregressive models (or decoder-only) in that they use self-attention, assuming all input tokens can relate to each other bidirectionally (hence the name BERT, from \textit{Bidirectional Embeddings Representations from Transformers}).

As shown in Figure~\ref{fig:Architectures}A, tokens encoded as word embeddings $E_i$ enter a self-attention mechanism without restricting connections, i.e., without using any masking.\footnote{It is important to distinguish a common term used to describe BERT-like models, known as \textit{masked language models}. In this case, BERT pre-trains by "masking" a set of tokens. However, this token masking does not correspond to masking the attention matrices but rather to replacing a token with a special label that allows predicting the token at an arbitrary position.} The BERT architecture uses multiple attention heads, allowing each head to specialize in learning particular types of relationships by learning appropriate attention weights, as shown in \cite{clark-etal-2019-bert}.

\subsection{GPT and Masked Attention}

The GPT model \cite{dai-etal-2019-transformer} is based on masking subsequent tokens in a natural language sentence. In Figure~\ref{fig:Architectures}B, token connections are limited to previous elements only. Like BERT, GPT is applied to language models but relies solely on the decoder of a transformer, predicting a token given the history of previous tokens. These models are also known as autoregressive models and can be thought of as generative models, as they are commonly used for natural language generation.

\subsection{T5 and Encoder-Decoder Attention}

Language models based on the T5 architecture \cite{raffel2020exploring} fully leverage the transformer architecture, integrating an encoder (based on self-attention) and a decoder (based on masked attention). The interaction between encoder and decoder is done through encoder-decoder attention. This enables the integration of input text with output text, allowing tasks such as machine translation or text summarization. This type of architecture functions similarly to the one proposed by Bahdanau et al. \cite{bahdanau2014neural}, although the latter uses recurrent layers in both encoder and decoder, also integrated via encoder-decoder attention. The essential idea is shown in Figure~\ref{fig:Architectures}C.

\subsection{Unlimiformer and Sparse Attention}

The Unlimiformer architecture \cite{bertsch2023unlimiformerlongrangetransformersunlimited} is used for text processing, but unlike other architectures, it uses a version of sparse attention where each attention head in the architecture considers only blocks of $k$ tokens selected via $k$-nearest neighbors. This set of neighbors forms the context to which attention layers are applied, while the other tokens are masked. Figure~\ref{fig:Architectures}D illustrates this idea for $k=2$. As seen, different attention heads can take different $k$-sized contexts, allowing each head to specialize in different regions. Likewise, each head can model specific sets of relations.

\begin{figure}
\centering
\includegraphics[width=0.85\linewidth]{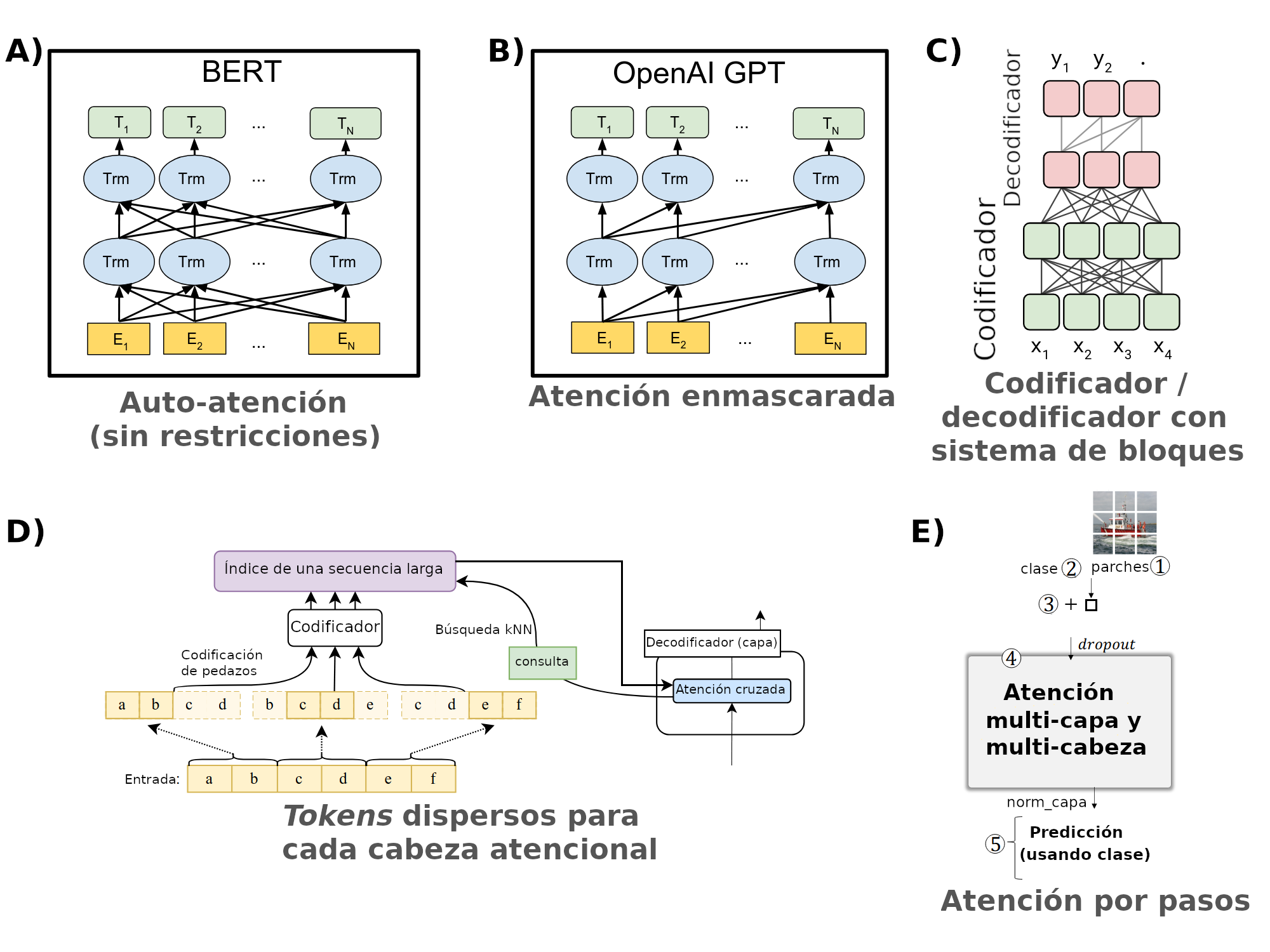}
\caption{Deep architectures using attention mechanisms: A) BERT (adapted from \cite{devlin2019bert}), B) GPT (adapted from \cite{dai-etal-2019-transformer}), C) T5 (adapted from \cite{raffel2020exploring}), D) Unlimiformer (adapted from \cite{bertsch2023unlimiformerlongrangetransformersunlimited}), E) ViT (adapted from \cite{Li_etal_ViT_explanation_attempt})}
\label{fig:Architectures}
\end{figure}

\subsection{ViT and Stride Attention}

The proposal of vision transformers or ViT \cite{dosovitskiy2021imageworth16x16words,Li_etal_ViT_explanation_attempt} is a transformer architecture for image processing. The idea is to treat images as token sequences by converting image patches, typically of size $16 \times 16$ pixels, into vectors that are then processed as a sequence through multiple attention heads. To capture spatial relationships in images via these patches, the ViT architecture employs stride attention to limit the scope of relationships that patches can have. This is visually illustrated in Figure~\ref{fig:Architectures}.

\section{Code for Attention Layers} \label{app:Code}

Below we present code implementations for the attention layers discussed in this article. These codes are written in Python using the PyTorch library. \footnote{A more complete version of the transformer code can be found at
\url{https://victormijangosdelacruz.github.io/MecanismosAtencion/}.}
All implementations linearly project the data to queries ($Q$), keys ($K$), and values ($V$).

\subsection{Self-Attention}

Self-attention is implemented via a dot product between $Q$ and $K$, followed by a softmax function:

\begin{verbatim}
    class SelfAttention(nn.Module):
        def __init__(self, d_model):
            super(SelfAttention, self).__init__()
            self.d_model = d_model
            self.Q = nn.Linear(d_model, d_model, bias=False)
            self.K  = nn.Linear(d_model, d_model, bias=False)
            self.V  = nn.Linear(d_model, d_model, bias=False)
            
        def forward(self, x):
            query,key,value = self.Q(x),self.K(x),self.V(x)
            scores = torch.matmul(query, key.T)/np.sqrt(self.d_model)
            p_attn = torch.nn.functional.softmax(scores, dim = -1)
            Vs = torch.matmul(p_attn, value).reshape(x.shape)
            
            return Vs, p_attn
\end{verbatim}

\subsection{Masked Attention}

For masked attention, the function \texttt{masked\_fill} is used with a large negative value $-1e9$ so that the softmax returns 0s in these entries. The function \texttt{masking} creates a lower-triangular masking matrix:

\begin{verbatim}
\begin{verbatim}
    class MaskAttention(nn.Module):
        def __init__(self, d_model):
            super(MaskAttention, self).__init__()
            self.d_model = d_model
            self.Q = nn.Linear(d_model, d_model, bias=False)
            self.K = nn.Linear(d_model, d_model, bias=False)
            self.V = nn.Linear(d_model, d_model, bias=False)
    
        def forward(self, x):
            query, key, value = self.Q(x), self.K(x), self.V(x)
            scores = torch.matmul(query, key.T)/np.sqrt(self.d_model)
            mask  = self.masking(x)
            scores = scores.masked_fill(mask == 0, -1e9)
            att = nn.functional.softmax(scores, dim=-1)
            h = torch.matmul(att, value)
    
            return h, att
    
        def masking(self, x):
            n = x.size(0)
            subsequent_mask = np.triu(np.ones((n, n)), k=1).astype('uint8')
            
            return torch.from_numpy(subsequent_mask) == 0
\end{verbatim}

\subsection{Encoder-Decoder Attention}

Encoder-decoder attention uses encoder elements as keys $K$ and values $V$, while the decoder provides the queries $Q$:

\begin{verbatim}
    class EncDecAttention(nn.Module):
        def __init__(self, d_model):
            super(EncDecAttention, self).__init__()
            self.d_model = d_model
            self.Q = nn.Linear(d_model, d_model, bias=False)
            self.K  = nn.Linear(d_model, d_model, bias=False)
            self.V  = nn.Linear(d_model, d_model, bias=False)
            
        def forward(self, x, encoder):
            query,key,value = self.Q(x),self.K(encoder),self.V(encoder)
            scores = torch.matmul(query, key.T)/np.sqrt(self.d_model)
            p_attn = torch.nn.functional.softmax(scores, dim = -1)
            h = torch.matmul(p_attn, value).reshape(x.shape)
            
            return h, p_attn
\end{verbatim}

\subsection{Stride Attention}

Stride attention is similar to masked attention, but the \texttt{masking} function also masks values outside the stride window:

\begin{verbatim}
    class SparseAttention(nn.Module):
        def __init__(self, d_model, stride=3):
            super(SparseAttention, self).__init__()
            self.d_model = d_model
            self.stride = stride
            self.Q = nn.Linear(d_model, d_model, bias=False)
            self.K = nn.Linear(d_model, d_model, bias=False)
            self.V = nn.Linear(d_model, d_model, bias=False)
    
        def forward(self, x):
            query, key, value = self.Q(x), self.K(x), self.V(x)
            scores = torch.matmul(query, key.T)/np.sqrt(self.d_model)
            mask  = self.masking(x)
            scores = scores.masked_fill(mask == 0, -1e9)
            att = nn.functional.softmax(scores, dim=-1)
            h = torch.matmul(att, value)
    
            return h, att
    
        def masking(self, x):
            n = x.size(0)
            mask = np.ones((n,n))
            for i in range(0,n):
                for j in range(0,self.stride):
                    m = max(0,i-j)
                    mask[i,m] = 0
            
            return torch.from_numpy(mask) == 0
\end{verbatim}

\subsection{Graph Attention}

Finally, a graph attention mechanism is defined using masking based on an adjacency matrix that defines data neighborhoods. This mechanism can simulate other attention types depending on the provided adjacency matrix:

\begin{verbatim}
    class GraphAttention(nn.Module):
        def __init__(self, d_model):
            super(GraphAttention, self).__init__()
            self.Q = nn.Linear(d_model, d_model, bias=False)
            self.K = nn.Linear(d_model, d_model, bias=False)
            self.V = nn.Linear(d_model, d_model, bias=False)
    
            self.d_model = d_model
            
        def forward(self, x, adj):       
            query, key, value = self.Q(x), self.K(x), self.V(x)        
            scores = torch.matmul(query, key.T)/np.sqrt(self.d_model)
            scores = scores.masked_fill(adj == 0, -1e9)
            p_attn = nn.functional.softmax(scores, dim = -1)
            h = torch.matmul(p_attn, value)
            
            return h, p_attn 
\end{verbatim}

\end{document}